\documentclass{article}
\usepackage{mathptmx}
\usepackage{soul}\setuldepth{article}

\usepackage{mdframed}

\usepackage{graphicx}
\usepackage{xcolor}
\usepackage{amsmath}
\usepackage{makecell}
\usepackage{amsmath}
\usepackage{amsthm}

\usepackage{booktabs} 
\usepackage{multirow}
\usepackage{array}
\usepackage[hidelinks]{hyperref}

\usepackage{amssymb}
\usepackage[english]{babel}

\usepackage[letterpaper,top=2cm,bottom=2cm,left=3cm,right=3cm,marginparwidth=1.75cm]{geometry}

\usepackage{amsmath}
\usepackage{graphicx}

\newtheorem{example}{Example}
\newtheorem{definition}{Definition}

\newtheorem{theorem}{Theorem}
\newtheorem{proposition}{Proposition}
\newtheorem{lemma}{Lemma}

\newcommand{\asp}{\text{ \it ASPIC$^+$ }} 
\newcommand{\lesac}{\text{\it LeSAC }}  
\newcommand{\lsc}{\text{\it LeSAC}}  
\newcommand{\prem}{\text{\tt Prem}} 
\newcommand{\conc}{\text{\tt Conc}} 
\newcommand{\sub}{\text{\tt Sub}} 
 
\newcommand{\sr}{\text{ \tt StRules}}

\newcommand{\norms}{\text{\tt Norms}}

\newcommand{\lp}{\text{\tt LastPrin}}

\title{Explaining Non-monotonic Normative Reasoning using Argumentation Theory with Deontic Logic}
\author{Zhe Yu\\Institute of Logic and Cognition, Department of Philosophy, Sun Yat-sen University \\\emph{yuzh28@mail.sysu.edu.cn}
\\Yiwei Lu
\\School of Law, Old College, University of Edinburgh}

\begin{document}
\maketitle

\begin{abstract}
	In our previous research, we provided a reasoning system (called \textit{LeSAC}) based on argumentation theory to provide legal support to designers during the design process. Building on this, this paper explores how to provide designers with effective explanations for their legally relevant design decisions. We extend the previous system for providing explanations by specifying norms and the key legal or ethical principles for justifying actions in normative contexts. Considering that first-order logic has strong expressive power, in the current paper we adopt a first-order deontic logic system with deontic operators and preferences. We illustrate the advantages and necessity of introducing deontic logic and designing explanations under \lesac by modelling two cases in the context of autonomous driving. 
	
	In particular, this paper also discusses the requirements of the updated \lesac to guarantee rationality, and proves that a well-defined \lesac can satisfy the rationality postulate for rule-based argumentation frameworks. This ensures the system's ability to provide coherent, legally valid explanations for complex design decisions.
\end{abstract}

\section{Introduction}\label{sec:intro}

Modelling and applying legal reasoning has been central to AI and Law technology. The goal is to ensure intelligent systems comply with the law, but translating laws, designed for judges with extensive context, into rules guiding machines with limited context is challenging. In previous research~\cite{lujurix22,lujurisin2023,lu2023legal}, we took a designer's perspective to support law-compliant design decisions. One key issue is ``deontological collapse", where the distinction between ``ought" and ``can refrain from" in human interpretation may be lost when translated into computational instructions. In the current paper, we introduce elements of deontic logic to address this issue and help explain design decisions in the midst of conflicting values.

This paper takes the legal context as a typical application context for normative reasoning. The following hypothetical example may not fully comply with the legal provisions and detailed requirements, but it may help to illustrate what we mean.

\begin{example}\label{exp:dog}
    A car has been involved in an accident with a driver and a pet dog in it, then who should we protect first? While individual emotional preferences may vary, the guidance of the law is clear: the driver should be protected first. More importantly, the law is able to give an adequate and acceptable legal explanation. For example, although both are lives, the safety of humans as citizens is prioritised over pets, which are legally defined as property. 
    
    However, What if the human driver is replaced with an autonomous driving system? Obviously, prioritising the protection of the driver, i.e. the vehicle's driving system, in this scenario would bring a lot of debates. Moreover, the law cannot gain public acceptance by simple compulsion, such as following past norms to require vehicles to protect the driving system. This is essentially because the original explanation of the law no longer applies. Autonomous driving systems are not citizens, nor are they lives. Regardless of which of the two the law ultimately deems should be prioritised for protection, a new and reasonable explanation needs to be given.

    And also, what if there's only one passenger on the car with a pet dog? Should we distribute the responsibilities of a driver to him in accidents? In such case should we require the passenger to keep sober like a human driver? These uncertainty and inconsistency above wont release the duty of designers to perform specialized designs, but add more stress to clarify why certain designs are made.
\end{example}

Legal explanation is essential for ensuring consistency, fairness, and public trust in legal practice. In legal disputes, both parties must present sound legal explanations for their claims and counter the opposing side's reasoning. Judges must also provide clear, understandable reasons for their rulings, such as the relevance of evidence, applicable legal provisions, and case facts. These explanations not only support communication and the expression of claims but also ensure that the justice system is transparent and trusted by the public. Unlike technical processes, such as facial recognition systems, people demand transparency in legal decisions. Without clear explanations, legal outcomes are less likely to be accepted by the public.

Legal explanations typically expected to include the following elements~(\cite{MILLER20191}): (1) Comparisons with alternative possibilities, particularly conflicting ones; (2) Selectivity, highlighting essential points and avoiding unnecessary detail; (3) Probabilities, though less convincing, should be supplemented with typical scenarios; (4) Sociability, reflecting the beliefs of the involved parties. These features help ensure that legal explanations are persuasive and understandable.

For autonomous vehicle (AV) designers and manufacturers, legal explanations play a crucial role in gaining trust and regulatory approval. However, they face unique challenges. Unlike traditional legal cases where explanations are needed after disputes arise, AV manufacturers need legal clarity during the design phase to avoid penalties and demonstrate compliance. The goal of legal explanations in this context is to show adherence to legal norms and public values, rather than to prevail in legal disputes. Furthermore, AV designers must navigate uncertainties in laws not yet fully adapted to automation. Explanations must address these inconsistencies between human-specific regulations and AVs.

To meet these challenges, we propose that legal explanations for AV design should include: (1) The core reasoning behind the accepted action; (2) Reasons for rejecting alternative actions, particularly conflicting ones; (3) The (legal or ethical) principles and values justifying the accepted action, such as prioritizing life over property; and (4) Clear statements of rights and obligations, indicating what is required, optional, or prohibited. Additionally, legal explanations should avoid relying solely on statistical probabilities, as these are insufficient for gaining trust in new, untested technologies.

Given the unique requirements for AV legal explanations, we see potential in applying the \lesac framework, which we have developed in previous research~\cite{lujurix22,lujurisin2023,lu2023legal}. To further enhance clarity, we introduce deontic logic operators to better express the underlying legal reasoning. The rest of this paper will be developed in the following format: in \S\ref{sec:lesac} we show how \lesac preforms reasoning after the introduction of the formal language based on first-order deontic logic; then we present the proofs of the rationality postulates. In \S\ref{sec:exp&case} we show how \lesac constructs explanations, with case studies to demonstrate our explanation.  Finally, in \S\ref{sec:conc} we conclude the paper and discuss future work.

\section{Non-monotonic Reasoning in Normative Contexts
}\label{sec:lesac}

As a basis for achieving non-monotonic reasoning, we first introduce the definition of \lsc.

\subsection{Logic and argumentation theory}\label{subsec:AT}
\textit{LeSAC}, which employs the basic mechanism of the structured argumentation framework $ASPIC^+$~\cite{prakken10,MP13} as its reasoning method, can adopt any specified logical language as the basis for knowledge representation. 
In previous work \cite{luyu20dl,lujurix22}, we translate the assertions ($ABox$) and the terminologies ($TBox$) of a Description Logic-based legal ontology as sets of knowledge and rules of the system\footnote{a mapping table is given in \cite{luyu20dl,yl22arxiv}}. 
Such a combination takes advantage of the argumentation system's ability to handle non-monotonic reasoning, allowing users to realize reasoning in uncertain and inconsistent legal contexts. \lesac can clearly reflect legal information of construction arguments, achieving conclusions, and which arguments defended/justified them. It also supports reasoning under preference based on priority orders of legal/ethical principles. 
It therefore has the potential to provide effective legal explanations under the criteria described in the previous section.

However, in the legal explanation of practical reasoning, legal norms generally have the feature of being defeasible, that is, the application of legal rules is uncertain. Only rules that can be logically deduced are guaranteed to be strict. Furthermore, manufacturers and designers will pay special attention to the lever of permissive or obligatory of certain actions. Description logics do not provide expressions for these aspects, therefore do not support to distinguish most defeasible rules from strict rules based on the logic system itself. It also lack the expression of permission and obligation. Both make the explanations given to designers and manufacturers not sufficient and clear enough.

Given the fact that deontic logic enable the formalization of what is obligatory to be done in ethical or legal contexts and provides tools for handling rules, permissions, and obligations in reasoning systems, and that first-order logic has strong expressive power, in the current paper we adopt a first-order deontic logic system with deontic operators ($O$ for obligation, $P$ for permission) and preferences. The formal language is defined inductively as follows ($t$ denotes a term and $\alpha$ denotes a formula):
\[
t ::= x
\mid c
\mid f(t_1, \ldots, t_n)
\]
where $x$ and $c$ represent variables and constants respectively, $f$ represents a function. Let $A$ represent a predicate and \( p \) an atomic formula ($A (t_1, \ldots, t_n)$). 
 \[\alpha ::= p
\mid\bot
\mid \neg\alpha  
\mid \alpha_1 \wedge \alpha_2 
\mid \alpha_1 \vee \alpha_2 
\mid \alpha_1 \rightarrow \alpha_2 
\mid \forall x \alpha  
\mid \exists x  \alpha   
\mid O \alpha
\mid P \alpha
\mid Pref(\alpha_1, \alpha_2)
\]
where 
\(\bot \) is falsehood, \( \wedge, \vee, \rightarrow, \neg \) are logical connectives, \( O \alpha \) represents that \( \alpha \) is obligatory, \( P \alpha \) represents that \( \alpha \) is permitted, and $Pref(\alpha_1, \alpha_2)$ expresses that \( \alpha_1 \) is preferred over \( \alpha_2 \).  $(\alpha_1 \rightarrow \alpha_2)\wedge(\alpha_2 \rightarrow \alpha_1)$ is denoted as $\alpha_1 \leftrightarrow \alpha_2$. 

 In order to express normative reasoning with preferences and to capture conflict relations reasonably, the following axioms and rules are necessary (including the basic rules of classical first-order logic).

    \begin{itemize}
        \item[$(1)$] Axiom Schemes: 
     \begin{equation*}
     \setlength{\arraycolsep}{10pt}
     \begin{array}{ll}
     \text{All axiom schemes of classical first-order logic.}     &\\
      (A1)~O(\alpha)\rightarrow P(\alpha)&
       (A2)~O(\alpha)\leftrightarrow \neg P(\neg\alpha)\\  
 (A3)~O(\alpha)\wedge O(\neg\alpha)\rightarrow\bot&
 (A4)~O(\alpha)\wedge O(\beta)\wedge Pref(\alpha, \beta)\rightarrow O(\alpha)\wedge\neg O(\beta)\\
 (A5)~Pref(\alpha, \beta)\wedge Pref(\beta, \gamma)\rightarrow Pref(\alpha, \gamma)&
 (A6)~Pref(\varphi, \psi)\rightarrow (O(\varphi)\rightarrow O(\psi))
     \end{array}     
     \end{equation*}
    
        \item[$(2)$] Rules of Inference:
            \[  \text{All rules of classical first-order logic.}
            \qquad
            (R1)~\frac{O(\alpha)\wedge(\alpha\rightarrow\beta)}{O(\beta)}    
            \qquad
            (R2)~\frac{O(\alpha)\wedge Pref(\alpha, \beta)}{\neg O(\beta)}
            \]           
    \end{itemize}

According to the axiom schemes, in addition to ($A3$), the following formulas also imply $\bot$:  $O(\alpha)\wedge\neg P(\alpha)$, $\neg P(\neg \alpha)\wedge O(\neg\alpha)$, $\neg P(\neg \alpha)\wedge \neg P(\alpha)$.  
Axiom ($A4$) expresses an \textit{Obligation Cancellation}, which means that (when there is a conflict between actions) the obligation for the less-preferred action should be cancelled, and Axiom ($A6$) means that if $\alpha$ is preferred over $\beta$, the obligation to do $\alpha$ should always take precedence if $\beta$ is obligatory. 

Let $\varDelta^L$ be a rule based legal knowledge base. 
Formally, a \lesac can be defined as follows.

\begin{definition}[LeSAC]\label{def:lesac}
A $LeSAC$ based on legal knowledge base $\varDelta^L$ is a five-tuple $(\mathcal{L}, \mathcal{R}, \mathcal{K}, \mathcal{P}, prin)$, where:
	\begin{itemize}
		\item $\mathcal{L}$ is a formal language with deontic operators. 
		\item $ \mathcal{R}=\mathcal{R}_{s}\cup \mathcal{N}$ is a set of rules such that $\mathcal{R}_{s}\cap \mathcal{N}=\emptyset $, where $\mathcal{R}_{s}$ is a set of strict inference rules of the form $ \varphi _{1},\ldots,\varphi_{n}\rightarrow \varphi $, and $ \mathcal{N} $ is a set of legal norms of the form $ \varphi _{1},\ldots,\varphi_{n}\Rightarrow \varphi $ ($\varphi _{i},\varphi \in\mathcal{L}$).
 \item $\mathcal{K}$ is a set of accepted (justified) knowledge based on $\varDelta^L$.
 \item $\mathcal{P}$ is a set of principles.
		\item $prin$ is a function from elements of $\mathcal{N}$ to elements of $\mathcal{P}$.
	\end{itemize}

\end{definition}

The set of strict rules $\mathcal{R}_s$ is derived from the logical formulas and rules we give, while norms are defeasible and only the set $\mathcal{N}$ can be prioritized, because strict rules express certain reasoning. 

Based on $\mathcal{K}$ arguments can be constructed (through rules). Let $\texttt{Prem}(A)$ return the set of all the formulas of $\mathcal{K}$ used to build an argument $A$, $\texttt{Conc}(A)$ return the conclusion of $A$, $\texttt{Sub}(A)$ return the set of all the subarguments of $A$, $\texttt{Norms}(A)$ return the set of all the norms applied in $A$, $\sr(A)$ return the set of all the strict rules (rules in $\mathcal{R}_s$) applied in $A$. 

Arguments constructed based on \lesac are defined as follows.

\begin{definition}[Arguments]\label{def:argument}
Let $\mathcal{A}$ be the set of all constructible arguments based on a $LeSAC$. Each argument $A\in\mathcal{A}$ takes one of the following forms: 
\begin{enumerate}
    \item $A= \varphi $, where $ \varphi \in \mathcal{K}$. In this case 
     $\prem(A)=\{\varphi\}$, $\conc(A)=\varphi$, $\norms(A)=\emptyset$,  $\sr(A)=\emptyset$, and $\sub(A)=\{\varphi\}$;
    \item $A= A_{1}$, $\ldots$, $A_{n}$ $\rightarrow \psi$, where $ A_{1} $, $ \ldots $, $A_{n}\in\mathcal{A}$ and there exists a strict rule $\rightarrow$ in $ \mathcal{R}_s$ such that $ \conc(A_{1}) $, $ \ldots $, $ \conc(A_{n}) $ $\rightarrow \psi$. In this case 
    $ \prem(A)=\prem(A_{1})\cup \ldots \cup \prem(A_{n}) $,  $\conc(A)=\psi$, 
    $ \norms(A)=\norms(A_{1})\cup \ldots \cup \norms(A_{n}) $,  
    $\sr(A)=\sr(A_1)\cup\ldots\cup\sr(A_n)\cup\{\conc(A_{1}) $, $ \ldots $ , $ \conc(A_{n}) $ $\rightarrow \psi\}$, and
    $ \sub(A)=\sub(A_{1})\cup \ldots \cup \sub(A_{n})\cup\{A\}$;  
    \item $A= A_{1}$, $\ldots$, $A_{n}$ $\Rightarrow \psi$, where $ A_{1} $, $ \ldots $, $A_{n}\in\mathcal{A}$ and there exists a defeasible rule (norm) $\Rightarrow$ in $\mathcal{N}$ such that $ \conc(A_{1}) $, $ \ldots $, $ \conc(A_{n}) $ $\Rightarrow \psi$. In this case 
    $ \prem(A)=\prem(A_{1})\cup \ldots \cup \prem(A_{n}) $, $\conc(A)=\psi$, 
    $ \norms(A)=\norms(A_{1})\cup \ldots \cup \norms(A_{n}) $ $\cup\{  \conc(A_{1}) $, $ \ldots $ , $ \conc(A_{n}) $ $\Rightarrow \psi\}$,  
    $\sr(A)=\sr(A_1)\cup\ldots\cup\sr(A_n)$,  and
    $ \sub(A)=\sub(A_{1})\cup \ldots \cup \sub(A_{n})\cup\{A\}$. 
\end{enumerate}

\end{definition}

Next we need to identify conflicts between arguments. Based on the given logical language and rules, as well as the application context of practical reasoning, conflicts could come from two aspects: 1. Logically identifiable conflicts; 2. Conflicts between decisions.

\begin{definition}[Conflicts]\label{def:conflict}
Let $A$, $B$, $B'\in \mathcal{A}$ be arguments, $A$ attacks $B$ on $B'\in\sub(B)$ of the form $B''_1, \ldots, B''_n\Rightarrow\varphi$\footnote{In cases where the consequence of a strict rule cannot be directly attacked, it is necessary to introduce the contraposition or transposition of the rule and construct an additional argument to attack its subargument~\cite{MP13}. Since the logical language adopted in this paper contains the axioms and rules of classical first-order logic, they will ensure that the transposition of strict rules also be in the rule set.} iff the following conditions hold: 
\begin{enumerate}
    \item $\conc(A)=-\varphi$\footnote{$ \psi=-\varphi$ denotes $ \psi=\neg\varphi $ or $ \varphi=\neg \psi$.} or $\conc(A)\wedge\conc(B')\rightarrow\bot$; 
    \item $\conc(A)=O(\alpha)$ and $\conc(B')=O(\beta)$, such that $O(\alpha)\wedge O(\beta)\rightarrow \neg O(\alpha)\vee\neg O(\beta)\in\mathcal{R}_s$.
\end{enumerate}
\end{definition}

If $A$ attacks $B$ on $B$, then we say that $A$ directly attacks $B$. 

Conflict identification is also necessary for consistency checking during product design, which is one of the main functions of \lesac in our previous paper.  

In Definition \ref{def:conflict}, condition 1 indicates that a conflict between two arguments arises from a contradiction between their conclusions. In addition, if the conclusions of two arguments lead to $\bot$, then there is also a conflict between them. 
Condition 2 states that if, in a particular legal context, it is not possible to ensure that both $\alpha$ and $\beta$ are taken, then there is a conflict between them.

Compared to \textit{ASPIC}$^+$'s definition of three types of attacks, we actually only identify the rebuttal attacks. The reasons for this are as follows: First, since the elements in the knowledge base are all accepted, they are not attacked in the current context, so there is no undermining attack on the premises. Second, in the legal context that emphasises argumentation, if the opponent wants to claim that a norm is not applicable, he must give reasons to support the claim. However, undercutting attacks are specified as ``preference-free''~\cite{MP13}, which seems to be groundless in legal contexts. This paper argues that such situation is equivalent to constructing a norm that declares that the premise of another norm does not support its conclusion, and provides the former with a higher priority legal principle.  That is, for a norm $\alpha\Rightarrow\beta$ (the principle on which it is based is $p_1$), another norm $\alpha\Rightarrow\neg\beta$  (the principle on which it is based is $p_2$) is given such that $p_2$ is strictly prior to $p_1$ ($p_1<p_2$).  Therefore, we do not include undercutting attacks either.

To identify justifiable arguments, conflicts might be resolved by preferences. 
The \textit{last-link principle}~\cite{MP13} is usually regarded to be more reasonable for preference acquisition in the context of normative reasoning. 

Let $A\in\mathcal{A}$ be an argument. 
$\texttt{LastNorms}(A)=\emptyset$ if $\texttt{Norms}(A)=\emptyset$, or $\texttt{LastNorms}(A)=\{\texttt{Conc}(A_1), \ldots, \texttt{Conc}(A_n)\Rightarrow\psi$\} if $A=A_1, \ldots, A_n\Rightarrow\psi$, otherwise $\texttt{LastNorms}(A)=\texttt{LastNorms}(A_1)\cup\ldots\cup \texttt{LastNorms}(A_n)$. We denote the corresponding set of principles as $\texttt{LastPrin}(A)$, i.e., $\{prin(n)|n\in\texttt{LastNorms}(A)\}$. 

Let $\leqslant$ be a preordering\footnote{We write $p<p'$ if and only if $p\leqslant p'$ and $p'\nleqslant p$; $p=p'$ if and only if $p\leqslant p'$ and $p'\leqslant p$.}
on $\mathcal{P}$ and $\triangleleft$ a set comparison, we present the following definitions that derive preferences from the priority orderings on principles.

\begin{definition}[Preferences on arguments]\label{def:prefarg}
For any $A$, $B\in\mathcal{A}$, $A\prec B$ iff $\lp(\beta)\triangleleft\lp(\alpha)$. 
$A\preceq B$ iff $A\preceq B$ or $\lp(A)=\lp(B)$.
\end{definition}

According to our definition of \lsc, if the users prefer to use the \textit{weakest-link principle}~\cite{MP13} instead of the \textit{last-link principle}, they can simply replace ``$\lp$'' with ``\texttt{Prin}'' (where $\texttt{Prin}(A)=\{prin(n)|n\in\texttt{Norms}(A)\}$) in the above definition.

\begin{definition}[Preferences on actions]\label{def:prefact}
    Let $A$, $B$ be arguments such that $\conc(A)=O(\alpha)$,  $\conc(B)=O(\beta)$, and $O(\alpha)\wedge O(\beta)\rightarrow\neg O(\alpha)\vee\neg O(\beta)$. $Pref(\beta, \alpha)\in\mathcal{K}$, 
    iff $B\nprec A$, $\nexists Pref( \alpha, \beta)\in\mathcal{K}$, and for some specific designers, $\beta$ takes precedence over $\alpha$. 
\end{definition}

For the first type of conflicts in Definition \ref{def:conflict}, we obtain preference between arguments according to priorities on principles corresponding to the norms. While for conflicts on action obligations, we obtain new rules (which are strict) through the $Pref$ relation between actions. Note that $\neg O(\alpha)\vee\neg O(\beta)\leftrightarrow P(\neg\alpha)\vee P(\neg\beta)$, indicating that in this case it is permissible to not take either action. 
This definition can be used to deal with the personal preferences of the designers and manufacturers. Nevertheless, to avoid legal or ethical violations, these personal preferences should only be used when they do not conflict with these principles or when the latter cannot resolve the conflict (in the case of $A\nprec B$ and $B\nprec A$). 
Through the $Pref$ relation, axiom scheme ($A4$) and rule ($R2$) infer that if one action is preferred in a conflict, the obligation to take the less preferred action, which cannot be performed simultaneously, should be cancelled. \textbf{Additionally, ($A6$) is indispensable, as it ensures that if the conflict cannot be resolved through preference between arguments, the obligation for the less preferred action will lead to a conflict and thus cannot be chosen. Therefore, even if a set of acceptable arguments concludes with the less preferred action, the corresponding actions will ultimately not be accepted.}

Given the set of all the arguments and the attack relations determined after preference comparison, a computable arguments evaluation process can be performed based on abstract argumentation frameworks (\textit{AF}s) and argumentation semantics~\cite{dung95} to output the set of justified arguments, according to which the set of accepted conclusions can be obtained. Since the \textit{complete semantics} is the most basic of the classical semantics (other semantics can be defined based on it), we introduce the following definition for justification of arguments based on the \textit{complete semantics}~\cite{dung95}. 

\begin{definition}[Argument evaluation]\label{def:AF}
Let $\langle \mathcal{A}, \mathcal{D}\rangle$ be an AF, where $\mathcal{A}$ is the set of all the arguments constructed based on a \lesac, and $\mathcal{D}$ is the set of defeats between arguments. 
For all arguments $A$, $B\in\mathcal{A}$, $(A, B)\in\mathcal{D}$ iff $A$ attacks $B$ and $A\nprec B$. A set of arguments  $E\subseteq\mathcal{A}$ is \textit{conflict-free} iff $\nexists A, B\in E$ such that $(A, B)\in \mathcal{D}$. An argument $A$ is said to be \textit{defended} by $E$, iff $\forall{B}\in \mathcal{A}$, if $(B, A)\in \mathcal{D} $, then $\exists C\in {E}$ such that $(C, {B})\in \mathcal{D}$. $E$ is said to be \textit{complete}, iff: 1) $E$ is  conflict-free, 2) $\forall A\in E$, $A$ is defended by $E$ and 3) $\forall A\in\mathcal{A}$ defended by $E$, $A\in E$.
\end{definition}

We say $A$ is justified/accepted with respect to $E$, if and only if $A\in E$. 

\subsection{Rationality requirements and postulates}\label{sec:RP}

In this subsection, we provide a detailed discussion on how to ensure that the settings of \lesac are reasonable and that the reasoning outcomes are rational, drawing on relevant literature from structured argumentation systems~\cite{MP13,modgil18,CA07}.

By Definition \ref{def:prefarg}, an ordering $\preceq$ on the set $\mathcal{A}$ based on a \lesac is reasonable if it satisfies the properties one expects~\cite{MP13}, depending on whether there are uncertain elements (i.e., norms in $\mathcal{N}$) contained in the constructed arguments. 
Following \cite{MP13,modgil18}, we introduce the concepts of \textit{reasonable inducing} and \textit{strict continuation} as follows.

\begin{description}
	\item[Reasonable inducing] Let $\texttt{NP}\in\{\texttt{Norms}, \texttt{LastNorms}, \texttt{Prin}, \texttt{LastPrin}\}$. A set ordering $\triangleleft$ is said to be a \textit{reasonable inducing} if it is irreflexive, transitive, and for all arguments $A$, $B_1$, $\ldots, B_n$ such that $\bigcup^n_{i=1}\texttt{NP}(B_i)\triangleleft \texttt{NP}(A)$, it holds that for some $i=1\ldots n$, $\texttt{NP}(B_i)\triangleleft \texttt{NP}(A)$. 
	\item[Strict continuation] For any set of arguments $\{A_1, \ldots, A_n\}$, $A'$ is a \textit{strict continuation} of $\{A_1, \ldots, A_n\}$ if and only if  $\norms(A')=\bigcup_{i=1}^n\norms(A_i)$, $\sr(A')\supseteq\bigcup_{i=1}^n\sr(A_i)$ and $\prem(A')\supseteq\bigcup_{i=1}^n\prem(A_i)$.
\end{description}

The following definition of reasonable argument ordering for \lesac is taken from ~\cite{MP13,modgil18}.

\begin{definition}[Reasonable preferences for arguments]\label{def-reasonableOrdering}
	Let $\preceq$ be a preference ordering on $\mathcal{A}$ constructed based on a \lsc.  $\preceq$ is reasonable iff the following two conditions are satisfied:
	\begin{enumerate}
		\item $\forall A, B\in\mathcal{A}$, 
		\begin{enumerate}
			\item if $\norms(A)=\emptyset$ and $B$ $\norms(B)\neq\emptyset$, then $B\prec A$;
			\item if $\norms(A)$ and $\norms(B)$ are both $\emptyset$, then $B\nprec A$;
			\item let $A'\in\mathcal{A}$ be a strict continuation of $\{A\}$; if $A\nprec B$, then $A'\nprec B$, and if $B\nprec A$, then $B\nprec A'$.
		\end{enumerate}
		\item for any sets of arguments $\{C_1, \ldots, C_n\}\subseteq\mathcal{A}$, let $C^{+\setminus i}$ denote some strict continuation of $\{C_1, \ldots, C_{i-1}, C_{i+1}, \ldots, C_n\}$; then, it cannot be the case that $\forall i$, $C^{+\setminus i}\prec C_i$.
	\end{enumerate}
\end{definition}

The following definition of set comparison, based on the \textit{Elitist} and \textit{Democratic} approaches~\cite{Cayrol92}, can be used as a reasonable set comparison method to back up Definition \ref{def:prefarg}.

\begin{definition}[Set Comparison]\label{def-eli-dem}
	Let $\varGamma$ and $\varGamma'$ be two finite sets and $s\in\{Eli, Dem\}$, $\triangleleft_s$ denotes a set comparison: 
	\begin{enumerate}
		\item if $\varGamma=\emptyset$ then $\varGamma\ntriangleleft_s\varGamma'$;
		\item if $\varGamma'=\emptyset$ and $\varGamma\neq\emptyset$, then $\varGamma\triangleleft_s\varGamma'$; otherwise,
		\item assuming a preordering $\leqslant$ over the elements in $\varGamma\cup\varGamma'$: 
		\begin{itemize}
			\item if $s=Eli$,  then $\varGamma\triangleleft_{Eli}\varGamma'$ if  $\exists X\in\varGamma$ such that $\forall Y\in\varGamma'$, $X< Y$; else
			\item if $s=Dem$, then $\varGamma\triangleleft_{Dem}\varGamma'$ if  $\forall X\in\varGamma$, $\exists Y\in\varGamma'$, $X< Y$. 
		\end{itemize}
	\end{enumerate}
	Then, $\varGamma\trianglelefteq_s\varGamma'$ iff $\varGamma\triangleleft_s\varGamma'$ or $\varGamma=\varGamma'$.\footnote{``='' denotes identity.}
\end{definition}

Let the $\preceq$ defined in Definition \ref{def:prefarg} be based on the $\triangleleft_s$ defined above, we present the following two propositions.

\begin{proposition}\label{pro-resonableOrder-last}
	The preference ordering $\preceq$ on arguments (according to the last-link principle)  is reasonable.
\end{proposition}

\begin{proposition}
	The preference ordering $\preceq$ on arguments (according to the weakest-link principle)  is reasonable.
\end{proposition}

Since \cite{MP13,modgil18} have already proven that $\trianglelefteq_s$ is reasonable inducing,  and by Definition \ref{def:prefarg}, our definition of preference on arguments ($\preceq$) is equivalent to $\trianglelefteq_s$, it is evident that $\preceq$ will also be reasonable. Therefore, we will not repeat the proofs here.

The following denotations are adapted from \cite{MP13}.
\begin{itemize}
	\item For any $Q\subseteq\mathcal{L}$, let $Cl_{\mathcal{R}_s}(Q)$ denote the closure of $Q$ under strict rules. 
	\item Let $Q\vdash\psi$ denote that there exists an argument $A$ constructed by only strict rules, such that $\texttt{Prem}(A)\subseteq Q$ and $\texttt{Conc}(A)=\psi$. 
\end{itemize}

Based on Definition \ref{def:conflict}, we present the following definition for \textit{direct} and \textit{indirect consistency}.

\begin{definition}[Consistency]\label{def:consistency}
	$Q\subseteq\mathcal{L}$ is \textit{directly consistent} iff $\nexists\varphi,\psi\in Q$ such that one of the following cases holds: 1) \(\varphi=-\psi\); 2) \(\varphi \wedge \psi\rightarrow \bot\); or 3) \(\varphi = O(\alpha)\), \(\psi = O(\beta)\) such that \(O(\alpha) \wedge O(\beta) \rightarrow \neg O(\alpha) \vee \neg O(\beta) \in \mathcal{R}_s\). \\
	$Q$ is \textit{indirectly consistent} if and only if $Cl_{\mathcal{R}_s}(Q)$ is directly consistent.
\end{definition}

In addition, a \lesac is said to be \textit{well-defined} if it satisfy the following conditions. 

\begin{definition}[Well-defined \lsc]\label{def-welld}
	A \lesac is well-defined, if it is 
	\begin{itemize}
		\item closed under contraposition or transposition, i.e., 
		\begin{enumerate}
			\item for all $Q\subseteq\mathcal{L}$ and $\varphi\in Q$, $\psi\in\mathcal{L}$, if $Q\vdash\psi$, then $Q\setminus\{\varphi\}\cup\{-\psi\}\vdash-\varphi$; or
			\item if $\varphi_1,\ldots, \varphi_n\rightarrow\psi\in\mathcal{R}_s$, then for each $i=1\ldots n$, there is \\
			$\varphi_1, \ldots, \varphi_{i-1},  -\psi, \varphi_{i+1}, \ldots \varphi_n\rightarrow-\varphi_i\in\mathcal{R}_s$;
		\end{enumerate}
		\item $Cl_{R_s}(\mathcal{K})$ is consistent; 
		\item for any minimal inconsistent $Q\subseteq\mathcal{L}$ and for any $\varphi\in Q$, it holds that $S\setminus\{\varphi\}\vdash-\varphi$. 
		\item \textbf{obligation cancellable}: $\forall A$, $B\in\mathcal{A}$ such that $\conc(A)\equiv\varphi$\footnote{This denotes $\conc(A)=\varphi$, or $\conc(A)=\gamma$ such that $\gamma\leftrightarrow\varphi$.} and $\conc(B)\equiv\psi$, if either \(\varphi \wedge \psi\rightarrow \bot\) holds, or \(\varphi = O(\alpha)\), \(\psi = O(\beta)\), and \(O(\alpha) \wedge O(\beta) \rightarrow \neg O(\alpha) \vee \neg O(\beta) \in \mathcal{R}_s\), then at least one of $A$, $B$ has a defeasible last rule. 
		(i.e., at least one of $\varphi$ or $\psi$ can only be the consequence of some norm in $\mathcal{N}$).
	\end{itemize}
\end{definition}

The first three bullets are adapted from \cite{MP13}. 
For the first bullet, since the logical language introduced in Section \ref{subsec:AT} includes all the axiom schemes and inference rules of classical first-order logic (e.g.,  
$(\neg\alpha\rightarrow\neg\beta)\rightarrow(\beta\rightarrow\alpha)$), based on which the set of strict rules of \lesac is formed, it is actually guaranteed that the strict rules are at least closed under transposition. 
For the second bullet, in the context of \lesac as defined in this paper, this requires that the set of justified beliefs excludes both direct and indirect inconsistencies. This should be ensured by the epistemic reasoning process before being input into \lsc. If structured argumentation theory is also applied to epistemic reasoning, a well-defined \asp framework, as proven in \cite{MP13}, can guarantee this.  
As for the third bullet, as well as the property of closure under contraposition in the first bullet, our logical language setting based on first-order logic can be defined to meet these requirements. In conclusion, the \lesac framework presented in this paper possesses the properties that satisfy the criteria of a well-defined rule-based argumentation theory. 

The last bullet is presented to ensure that at least one of two conflicting obligations can be cancelled without creating an inconsistency in the conclusion. 

In \cite{CA07} Caminada and Amgoud declare four basic \textit{rationality postulates} that any rule-based argumentation formalisms should at least fulfil, namely, \textbf{subargument closure}, \textbf{closure under strict rules}, \textbf{direct consistency}, and \textbf{indirect consistency}.
We present the following lemmas and proposition to specify that a well-defined  \lesac satisfies these basic rationality postulates under the complete semantics.

\begin{lemma}\label{lem-ABdefeat}
	For any $A, B\in\mathcal{A}$, 
	\begin{enumerate}
		\item if $A'\in \texttt{Sub}(A)$ such that $(B, A')\in \mathcal{D}$, then $(B, A)\in \mathcal{D}$; 
		\item 	if $\norms(A)=\emptyset$ and $A$ attacks $B$, then $(A, B)\in \mathcal{D}$;
		\item if $A$, $B$ directly attack each other, then at least one of the following cases holds: \\
		i) $(A, B)\in \mathcal{D}$; ii) $(B, A)\in \mathcal{D}$.
	\end{enumerate}
\end{lemma}
\begin{proof}
	
	1. Straightforward according to Definition \ref{def:conflict}.
	
	2. By Definition \ref{def:prefarg}, \ref{def-eli-dem} and Proposition \ref{pro-resonableOrder-last}, since $\triangleleft_s$ is reasonable inducing, $\preceq$ on arguments based on which is also reasonable. Therefore, in this case it is impossible $A\prec B$. Then, by Definition \ref{def:AF}, $(A, B)\in \mathcal{D}$. 
	
	3. Since $A$ and $B$ directly attack each other, according to Definition \ref{def:conflict}, both $A$ and $B$ are of the form $C_1, \ldots, C_n \Rightarrow \varphi$, where $C_1, \ldots, C_n$ are arguments in $\mathcal{A}$, and the last inference rules are norms. According to Definition \ref{def:prefarg}, the preference between them is determined by the preordering $\leqslant$ between the principles associated with the last norms of $A$ and $B$. It is impossible for both $A \nprec B$ and $B \nprec A$ to hold. Therefore, by Definition~\ref{def:AF}, at least one of $(A, B) \in \mathcal{D}$ or $(B, A) \in \mathcal{D}$ must hold.
\end{proof}

\begin{lemma}\label{lem-cA}
	For any complete extension $E=\{A_{1}$ , $\ldots$, $A_{n}\}$, if there exists an argument $A$ such that $A$ is a strict continuation of $E$, then $A\in E$.
\end{lemma}
\begin{proof}
	Suppose for contradiction. Since $E$ is a complete extension, either $A$ is not defended by $E$ or $E\cup \{A\}$ is not conflict-free.
	\begin{enumerate}
		\item If $A$ is not defended by $E$, then $\exists B\in\mathcal{A}$ such that $B$ defeats $A$, and $\nexists A_i\in E$ such that $A_i$ defeats $B$. Since $A$ is a strict continuation of $E$, it is impossible for $A$ to be of the form $A'_1, \ldots, A'_n\Rightarrow \varphi$. Therefore, by Definition \ref{def:conflict}, $B$ cannot directly attack $A$. Consequently,  $B$ attacks and defeats $A$ on $A_1/\ldots/A_n\in E$, which contradicts that $E$ is complete, as every argument in $E$ must be defended by $E$.
		\item If $E\cup \{A\}$ is not conflict-free, then $\exists A_i\in E$ such that $(A, A_i)\in \mathcal{D}$ or $(A_i, A)\in \mathcal{D}$, both cases lead to a contradiction to $E$ is conflict-free: 1) if $(A, A_i)\in \mathcal{D}$, then since $A_i$ is defended by $E$, $\exists A_j\in E$ such that $A_j$ defeats $A$; this situation is similar to 1 and $A_j$ must defeat at least one of $A_1, \ldots, A_n$; 2) if $(A_i, A)\in \mathcal{D}$, then $A_i$ must also defeat at least one of $A_1, \ldots, A_n$ for defeating $A$. In either case, $E$ is not conflict-free.
	\end{enumerate}
	Since $A$ is defended by $E$ and $E\cup \{A\}$ is conflict-free, $A\in E$.
\end{proof}

\begin{theorem}\label{pr:RP}
Let $E$ be a complete extension based on \lesac. The following four properties hold.
	\begin{itemize}
		\item $\forall A\in E$, if $A'\in \texttt{Sub}(A)$, then $A'\in E$. (Sub-argument Closure)
		\item $\{\texttt{Conc}(A)|A\in E\}=Cl_{\mathcal{R}_s}(\{\texttt{Conc}(A)|A\in E\})$. (Closure under Strict Rules)
		\item $\{\texttt{Conc}(A)|A\in E\}$ is consistent. (Direct Consistency)
		\item $Cl_{\mathcal{R}_s}(\{\texttt{Conc}(A)|A\in E\})$ is consistent. (Indirect Consistency)
	\end{itemize}
\end{theorem}
\begin{proof}
	\textit{Proof of the first bullet.}
	
	Suppose $\exists A'\in \texttt{Sub}(A)$ such that $A'\notin E$. Since $E$ is a complete extension, either $A'$ is not defended by $E$, or $\{A'\}\cup E$ is not conflict-free, both lead to a contradiction:

1.  if $A'$ is not  defended by $E$, then $\exists B\in\mathcal{A}$ such that $(B, A')\in \mathcal{D}$ and $\nexists C\in E$ such that $(C, B)\in \mathcal{D}$. According to Lemma \ref{lem-ABdefeat}-1, $(B, A)\in \mathcal{D}$; then $A$ is not defended by $E$, contradicting each argument in $E$ is defended by $E$;

2.  if $\{A'\}\cup E$ is not conflict-free, then $\exists B\in E$ such that $(B, A')\in \mathcal{D}$. According to Lemma \ref{lem-ABdefeat}-1, $(B, A)\in \mathcal{D}$, contradicting $E$ is conflict-free.

	\textit{Proof of the second bullet.}
	
	Statement in this bullet is equivalent to saying that given $E=\{A_1, A_2, \ldots, A_n\}$, $\forall A'$ such that $\sr(A')\supseteq\bigcup_{i=1}^n\sr(A_i)$, with $\prem(A')=\bigcup_{i=1}^n\prem(A_i)$, $\norms(A')=\bigcup_{i=1}^n\norms(A_i)$, we have $A'\in E$. Therefore, $A'$ is  a strict continuation of $E$, and since $E$ is complete, this has already been proven by Lemma \ref{lem-cA}.

	\textit{Proof of the third bullet.}
	
	Suppose $\{\texttt{Conc}(A)|A\in E\}$ is inconsistent, i.e., $\exists A, B\in E$ such that $\texttt{Conc}(A)=-\texttt{Conc}(B)$, or \(\texttt{Conc}(A) \wedge \texttt{Conc}(B)\rightarrow \bot\); or \(\texttt{Conc}(A)= O(\alpha)\), \(\texttt{Conc}(B) = O(\beta)\), and \(O(\alpha) \wedge O(\beta) \rightarrow \neg O(\alpha) \vee \neg O(\beta) \in \mathcal{R}_s\).

	We show that each possible case leads to a contradiction.
	\begin{enumerate}
		\item If $\norms(A)=\emptyset$ and $\norms(B)=\emptyset$, then by our setting for \lesac, $\texttt{Conc}(A), \texttt{Conc}(B)\in Cl_{\mathcal{R}_s}(\texttt{Prem}(A)\cup \texttt{Prem}(B))$, contradicting \textit{axiom consistency} for a well-defined \lesac in Definition \ref{def-welld}. (This also violates the requirement of \textit{obligation cancellation}).
		\item\label{prf-2} If $\norms(A)=\emptyset$,  $\norms(B)\neq\emptyset$, and 
		\begin{enumerate}
			\item\label{prf-2a} 
			the last rule of $B$ is defeasible, then $A$ directly attacks $B$. 
			By Lemma~\ref{lem-ABdefeat}-2, $(A, B)\in \mathcal{D}$,  contradicting $E$ is conflict-free. 
			\item\label{prf-2b} 
			the last rule of $B$ is strict, then $A$ cannot directly attack $B$. The following cases of attacks can be further distinguished:
			\begin{enumerate}
				\item $\texttt{Conc}(A)=-\texttt{Conc}(B)$. Let $B'_1, \ldots, B'_n\in\sub(B)$ such that $\{\conc(B'_1), \ldots, \conc(B'_n)\}\vdash\{\conc(B)\}$, and $B'_i$ ($i=1\ldots n$) one of the arguments such that the last rule of $B'_i$ is defeasible. Since the logic language guarantees that transposed rules are all in the set $\mathcal{R}_s$, there must exist strict rules such that $\{\conc(A)\}\vdash-\conc(B'_i)$. Let $A'$ be the strict continuation of $\{A\}$ such that $\conc(A')=-\conc(B'_i)$. Hence, $A'$ attacks $B$ on $B'_i$. Similar as in case~\ref{prf-2a}, $(A', B)\in\mathcal{D}$. Since $A'\in E$, this contradicts that $E$ is conflict-free. 
				\item \(\texttt{Conc}(A) \wedge \texttt{Conc}(B)\rightarrow \bot\). By  \textit{obligation cancellation} in Definition \ref{def-welld}, the only possible case is $\exists B'\in\sub(B)$ such that $\conc(B')=\gamma$, where $\gamma\leftrightarrow\conc(B)$ and the last rule of $B'$ is defeasible. Due to the setting of $\bot$, it must be \(\texttt{Conc}(A) \wedge \texttt{Conc}(B')\rightarrow \bot\). Therefore, $A$ attacks and defeats $B$ on $B'$, contradicting that $E$ is conflict-free. 
				\item \(\texttt{Conc}(B) = O(\beta)\), and \(O(\alpha) \wedge O(\beta) \rightarrow \neg O(\alpha) \vee \neg O(\beta) \in \mathcal{R}_s\). By  \textit{obligation cancellation} in Definition \ref{def-welld}, the only possible case is $\exists B'\in\sub(B)$ such that $\conc(B')=\neg P(\neg\beta)$ and the last rule of $B'$ is defeasible. By the rule \(O(\alpha) \wedge O(\beta) \rightarrow \neg O(\alpha) \vee \neg O(\beta) \), we can get \(O(\alpha)\rightarrow P(\neg\beta)\in \mathcal{R}_s\). Let $A'$ be the strict continuation of $A$ constructed by this rule. Therefore, $A'$ attacks and defeats $B$ on $B'$. Since $A'\in E$, this contradicts that $E$ is conflict-free. 
			\end{enumerate}
		\end{enumerate}
		\item If $\norms(A)\neq\emptyset$ and  $\norms(B)=\emptyset$, the situation is similar to \ref{prf-2}, except that $A$, $B$ should be swapped. 
			
		\item If $\norms(A)\neq\emptyset$,  $\norms(B)\neq\emptyset$, and
		\begin{enumerate}
			\item the last rules of $A$, $B$ are both defeasible, then $A$, $B$ directly attacks each other.  By Lemma~\ref{lem-ABdefeat}-3, at least one of $(A, B)\in \mathcal{D}$ and $(B, A)\in \mathcal{D}$ holds. Both cases contradict $E$ and are conflict-free.
			\item the last rules of $B$ is strict, then the case is like \ref{prf-2b}, except that $A'$/$A$ and $B'_i$/$B$ attack each other instead. Thus,  by Lemma~\ref{lem-ABdefeat}-3, at least one of $(A'/A, B'_i/B)\in \mathcal{D}$ and $(B'_i/B, A'/A)\in \mathcal{D}$ holds; all of these cases contradict that $E$ is conflict-free.
		\end{enumerate}
			\end{enumerate}
	\textit{Proof of the fourth bullet.}
	
	This follows from the properties \textit{closure under strict rules} and the \textit{direct consistency}.
\end{proof}

\section{Explanation in Normative Contexts}\label{sec:exp&case}

\subsection{Formal explanation}\label{subsec:exp}
Based on the above definitions of \lsc, we present the following definition of why a certain decision is accepted in a normative context. 

Let $E_{Co}$ denote a complete extension of arguments obtained by Definition \ref{def:AF}.

\begin{definition}[Explanation]\label{def:XAI}
	Let $A$, $B\in\mathcal{A}$ be arguments, $\leqslant$ the priority orderings on $\mathcal{P}$, and $\alpha$ a formula of the logic language. 
	A normative \textit{explanation} for accepting $\alpha$ (under complete semantics) is $Exp(\alpha)=\mathcal{C}(A)\cup\mathcal{C}(B)\cup\{<_{l}\}$, where $A$, $B\in E_{Co}$, and: 
	\begin{itemize}
		\item $\conc(A)=\alpha$ and $\mathcal{C}(A)=\norms(A)\cup\lp(A)$;
		\item $B$ defends $A$, and $\mathcal{C}(B)=\norms(B)\cup\lp(B)$;
		\item 
		$\forall p, p'\in \lp(A)\cup \lp(B)$, $p<_{l}p'$ iff $p< p'$. 
	\end{itemize}
\end{definition}

The first bullet intended to explain how $\alpha$ is reached based on legal knowledge base $\varDelta^L$. The second bullet intended to explain why $\alpha$ is justified (defended). The sets of last principles indicate the principles that are most relevant to the conclusions of these arguments (if they are normative arguments), while the corresponding priority orderings indicate the basis for the most relevant preferences of these arguments. 

Since all applied norms are shown, their antecedents must all exist in the knowledge base $\mathcal{K}$ or be implied (strictly inferable) by it, therefore they do not necessarily have to be displayed repeatedly. 
This explanation is primarily intended to outline the basic elements required for a normative reasoning explanation; if the user requires more information, such as the set of all principles associated with argument $A$, the system can also return the sets containing these elements.

\subsection{Case demonstrations}\label{subsec:case}

We illustrate the advantages and necessity of introducing deontic logic and designing  explanations under \lesac by modelling the case mentioned in \S \ref{sec:intro}. In this example, the autonomous car has a only passenger with a pet dog. And, we assume that future laws require the only passenger to take on the responsibility of human driver in an accident.\footnote{For the convenience of the reader, formulas of the statements are written to appear self-explanatory. In the Example~\ref{exp:model}, protecting the dog can be understood as taking actions that prioritise the dog's safety, such as steering in a direction that avoids hitting the dog directly in the event of an impending collision.}

\begin{example}[Modelling of Example \ref{exp:dog}]\label{exp:model}
    ~\\   
Set of legal principles:
\begin{footnotesize}
\begin{equation*}
\mathcal{P}
    =\left\{
\begin{array}{l}
    p_1: \text{ The safety of citizens' lives must be protected.} \\
    p_2: \text{ Citizens' property should be protected.} \\
    p_3: \text{ Life should be protected.}\\
    p_4: \text{ The only passenger of an AV should be considered to have the same responsibilities as the driver.}\\
    p_5: \text{ Drivers have to guarantee the ability to control the car and watch the road conditions for safety.}\\
    p_6: \text{ Passengers do not bear the driver's responsibility.}
   
\end{array}
 \right\}
\end{equation*}
\end{footnotesize}
\\
Type rules: 
\begin{footnotesize}
\begin{equation*}
\begin{array}{ll}
  r_1: \forall x(Citizen(x)\wedge Life(x)\Rightarrow O(Protect(x)) ~~~~~  &
  r_2: \forall x(Property(x)\Rightarrow O(Protect(x))  \\
  r_3: \forall x(Life(x)\Rightarrow O(Protect(x))   &
  r_4: \exists x(OnlyPassenger(x))\Rightarrow O(Driver(x))\\
 r_5: \forall x(Driver(x)\Rightarrow O(Sober(x))&
 r_6: \forall x(Passenger(x)\Rightarrow P(\neg Sober(x))\\
  r_7: \forall x(Driver(x)\Rightarrow Life(x))&
  r_8: \forall x(Dog(x)\Rightarrow Life(x))\\
  r_9: \forall x(Dog(x)\Rightarrow Property(x))&\\
\end{array}
\end{equation*}
\end{footnotesize}
\\
Assume the driver or passenger's name is Roger and the dog's name is Pongo. After introducing constants based on knowledge ($\mathcal{K}$) in this context, we get the following set of token rules (i.e., norms in $\mathcal{N}$).
\begin{footnotesize}
\begin{equation*}
\mathcal{N}
    =\left\{
\begin{array}{ll}
  n_1: Citizen(Roger)\wedge Life(Roger)\Rightarrow O(Protect(Roger)) &
  n_2: Property(Pongo)\Rightarrow O(Protect(Pongo)) \\
  n_3: Life(Roger)\Rightarrow O(Protect(Roger))   &
  n_4: OnlyPassenger(Roger)\Rightarrow O(Driver(Roger))\\
  n_5: Driver(Roger)\Rightarrow O(Sober(Roger))
  &
  n_6: Passenger(Roger)\Rightarrow P(\neg Sober(Roger))\\
  n_7: Driver(Roger)\Rightarrow Citizen(Roger)\wedge Life(Roger)&
  n_8: Dog(Pongo)\Rightarrow Life(Pongo)\\
  n_9: Dog(Pongo)\Rightarrow Property(Pongo)&
  n_{10}: Life(Pongo)\Rightarrow O(Protect(Pongo)) \\
\end{array}
 \right\}
\end{equation*}
\end{footnotesize}

The correspondence between norms and principles including:
\begin{footnotesize}
\[prin(n_1)=p_1;~prin(n_2)=p_2;~prin(n_3)=p_3; ~prin(n_4)=p_4;~prin(n_5)=p_5;~prin(n_6)=p_6\]
\end{footnotesize}

In line with legal intuition and the idea that classification comes first, the following orderings among the principles can be considered reasonable: $p_2<p_3<p_1$, $p_6<p_5<p_4$. Accordingly, there are $n_2<n_3<n_1$ and $n_6<n_5<n_4$. Based on Definition \ref{def:prefact}, $Pref(Protect(Roger), Protect(Pongo)$ is in the knowledge based $\mathcal{K}$ and strict rules $r_{2}$ and $r_{3}$  can be constructed. Meanwhile,  $r_1$ is constructed based on the Axiom scheme (A2).  
\begin{footnotesize}
\begin{equation*}
\mathcal{R}_s    
=\left\{
\begin{array}{ll}
   r_1: O(Sober(Roger))\rightarrow \neg P(\neg Sober(Roger))&\\
   r_2: O(Protect(Roger)\wedge O(Protect(Pongo))\rightarrow \neg O(Protect(Roger))\vee\neg O(Protect(Pongo)) \\
   r_3: O(Protect(Roger)\wedge O(Protect(Pongo))\wedge Pref(Protect(Roger), Protect(Pongo))\rightarrow O(Protect(Roger))\wedge\neg O(Protect(Pongo)) \\
   r_4: O(Protect(Roger)\wedge Pref(Protect(Roger), Protect(Pongo))\rightarrow \neg O(Protect(Pongo))\\
      r_5: Pref(Protect(Roger), Protect(Pongo)\rightarrow (O(Protect(Pongo))\rightarrow O(Protect(Roger))
\end{array}
 \right\}
\end{equation*}
\end{footnotesize}

According to preference on arguments, the argument with the conclusion $P(\neg Sober(Roger))$ will be defeated; while based on $r_2$, an argument can be constructed to defend $O(Protect(Roger))$, so that it will apear in the final output of accepted conclusions.
\end{example}

By Definition \ref{def:XAI}, the explanations for accepting conclusions ``$O(Sober(Roger))$'' and ``$O(Protect(Roger)$'' are: 
\begin{footnotesize} 
\begin{equation*}
Exp(O(Sober(Roger)))    
=\left\{
\begin{array}{ll}
   n_4:~OnlyPassenger(Roger)\Rightarrow O(Driver(Roger))&\\
   n_5:~Driver(Roger)\Rightarrow O(Sober(Roger))&\\
   p_5:~\text{A driver must keep sober while driving.}
\end{array}
 \right\}
\end{equation*}\\
 \begin{equation*}
 Exp(O(Protect(Roger)))    
=\left\{
\begin{array}{ll}
   n_1:~Citizen(Roger)\wedge Life(Roger)\Rightarrow O(Protect(Roger))&\\
   n_7: Driver(Roger)\Rightarrow Citizen(Roger)\wedge Life(Roger)\\
   p_1:~\text{The safety of citizens' lives must be protected.}\\
\end{array}
 \right\}
 \end{equation*}
 \end{footnotesize}

For the argument defending ``$O(Sober(Roger))$'', the last norm applied in it is $n_5$, the same as in the argument concluding it; while the last rule applied in it is the strict rule $r_1$, due to the fact that $O(Sober(Roger))$ implies $\neg P(\neg Sober(Roger))$ according to the Axiom Scheme ($A2$). Therefore, there is only one last principle involved, i.e. $p_5$.

For the conclusion $O(Protect(Roger)$, the norms leading to it include $n_1$ and $n_7$. As for the conflicting action $O(Protect(Pongo))$, the argument leading to it is strictly defeated since $p_1$ has the highest priority.

 Therefore, these two explanations could be legally understood as: (1) The only passenger Roger has to be sober in the car. Because the law rules that the only passenger in the autonomous car has to take the responsibility of the driver. And a driver has to be sober. The underlying legal principle is drivers have to guarantee the ability to control the car and watch the road conditions for safety. (2) In the accident, the passenger Roger has to be protected. Because the law rules that citizens' life has to be protected. 
 We can see these explanations contain the elements we discussed before and are very easy to be understood.

It's also worth noting that, in our language of logic, the cancellation of an obligation doesn't mean that the action is forbidden. In this case, it doesn't mean that we are not allowed to save the dog in any case ($O(\neg Protect(Pongo))$). 

Furthermore, if we assume that in the future Roger the person will be replaced by an autonomous vehicle system, then, for example, the following norms can be added: $n'_1: AIdriver(Roger)\Rightarrow\neg Life(Roger)$, $n'_2: AIdriver(Roger)\Rightarrow Property(Roger)$. Thus, the sub-argument of the argument that leads to $O(Pretect(Roger))$ will be attacked, and according to the ordering of legal principles $  p_2<p_3<p_1$, the action $O(Protect(Pongo))$ will be more preferred because of $Life(Pongo)$. 

Since all the conflicts in the above cases can be resolved by the priority ordering over principles, let us further consider the following legal case in the context of autonomous driving.

\begin{example}\label{car}
	An autonomous vehicle carrying passengers encounters a dog in the road. And alongside the road are trees and street lamps. We assume that the vehicle is equipped with the ability to recognise animals. Furthermore, we assume it can correctly recognise the animal and  judge that if it brakes suddenly, the passengers may be in danger. If it swerves and hits the street lamp, the lamp will be damaged. If it hits the tree, the passengers will not be in danger, but the vehicle itself will be damaged. What should it do? And more importantly, why?
\end{example}

\begin{example}[Modelling of Example \ref{car}]\label{exp:model2}
	~\\   
	Set of legal and ethical principles:
	\begin{footnotesize}
		\begin{equation*}
			\mathcal{P}
			=\left\{
			\begin{array}{l}
				p_1: \text{ The safety of citizens' lives has the highest priority while driving. (legally)} \\
				p_2: \text{ While safe, the value of an animal's life should be higher than the value of property. (ethically)}\\
				p_3: \text{ It is not allowed to deliberately harm living animals. (legally)}\\
				p_4: \text{ Passengers enjoy equal clause protection as citizens on the vehicle. (legally)}\\
			\end{array}
			\right\}
		\end{equation*}
	\end{footnotesize}
	\\
	Assume that the car is AV, Roger is the passenger, and Pongo is the dog. After introducing constants based on knowledge ($\mathcal{K}$) in this context, we get the following set of norms. 
	\begin{footnotesize}
		\begin{equation*}
			\mathcal{N}
			=\left\{
			\begin{array}{ll}
				n_1: Citizen(Roger)\wedge Life(Roger)\Rightarrow O(KeepSafe(Roger)) &\\
				n_2: O(KeepSafe(Roger))\Rightarrow O(\neg HardBrake(AV)) \\
				n_3: Life(Pongo)\Rightarrow O(\neg Hurt(Pongo))   &\\
				n_4: O(\neg Hurt(Pongo))\Rightarrow O(HardBrake(AV))\\
				n_5: O(\neg Hurt(Pongo))\wedge P(\neg HardBrake(AV))\Rightarrow O(Swerve(AV))
				&\\
				n_6: O(Swerve
				(AV))\Rightarrow O(HitTree(AV)\wedge \neg Damage(Lamp))\\
				n_7: O(Swerve(AV))\Rightarrow O(HitLamp(AV)\wedge \neg Damage(AV))&\\
				n_8: Dog(Pongo)\Rightarrow Life(Pongo)\\
				n_9: Passenger(Roger)\Rightarrow Citizen(Roger)\wedge Life(Roger)&\\
			\end{array}
			\right\}
		\end{equation*}
	\end{footnotesize}
	
	The correspondence between norms and principles including:
	\begin{footnotesize}
		\[prin(n_{1-2})=p_1;~prin(n_{3-7})=p_2; ~prin(n_8)=p_3;~prin(n_9)=p_4\]
	\end{footnotesize}
	
	In line with legal and ethical intuition, the following orderings among the principles can be considered reasonable: $p_2<p_3<p_1$, while $p_3\approx p_4$. Accordingly, (for norms with conflicting consequents) there is $n_4<n_2$. 
	In addition, the actions introduced by norms $r_6$ and $r_7$ cannot be taken at the same time (reflected by the following $r_2$). Assuming that, in this case, designers and manufacturers prefer not to damage public property such as street lamps in order to avoid causing more problems (while cars can be warranted by manufacturers or reimbursed by insurance companies), we can obtain the following $r_3$ and $r_4$ based on Definition \ref{def:prefact}.  Meanwhile, strict rule $r_1$ is constructed based on the Axiom scheme ($A1$).  
	\begin{footnotesize}
		\begin{equation*}
			\mathcal{R}_s    
			=\left\{
			\begin{array}{ll}
				r_1:  O(\neg HardBrake(AV))\rightarrow P(\neg HardBrake(AV))&\\
				r_2: O(HitTree(AV)\wedge \neg Damage(Lamp))\wedge O(HitLamp(AV)\wedge \neg Damage(AV)) \\
				~~~~~\rightarrow  \neg O(HitTree(AV)\wedge \neg Damage(Lamp))\vee \neg O(HitLamp(AV)\wedge \neg Damage(AV))\\
				r_3: O(HitTree(AV)\wedge \neg Damage(Lamp))\wedge O(HitLamp(AV)\wedge \neg Damage(AV)) \\~~~~~\wedge Pref(HitTree(AV)\wedge \neg Damage(Lamp), HitLamp(AV)\wedge \neg Damage(AV))\\~~~~~\rightarrow O(HitTree(AV)\wedge \neg Damage(Lamp))\wedge\neg O(HitLamp(AV)\wedge \neg Damage(AV)) \\
				r_4: O(HitTree(AV)\wedge \neg Damage(Lamp)) \wedge Pref(HitTree(AV)\wedge \neg Damage(Lamp), HitLamp(AV)\wedge \neg Damage(AV))\\~~~~~\rightarrow\neg O(HitLamp(AV)\wedge \neg Damage(AV)) \\
				r_5: Pref(HitTree(AV)\wedge \neg Damage(Lamp), HitLamp(AV)\wedge \neg Damage(AV))\rightarrow (O(HitLamp(AV)\wedge \neg Damage(AV))\\~~~~~\rightarrow O(HitTree(AV)\wedge \neg Damage(Lamp))
			\end{array}
			\right\}
		\end{equation*}
	\end{footnotesize}
	
	According to preference on arguments, the argument with the conclusion ``$O(HardBrake(AV))$'' will be defeated; while based on $r_3$, an argument can be constructed to defend ``$O(HitTree(AV)\wedge \neg Damage(Lamp))$'', so that it will appear in the final output of accepted conclusions.
\end{example}

By Definition \ref{def:XAI}, the explanations for accepting conclusions ``$O(\neg HardBrake(AV))$'' and ``$O(HitTree(AV)\wedge \neg Damage(Lamp))$'' are: 
\begin{footnotesize} 
	\begin{equation*}
		Exp(O(\neg HardBrake(AV)))   
		=\left\{
		\begin{array}{ll}
			n_1: ~Citizen(Roger)\wedge Life(Roger)\Rightarrow O(KeepSafe(Roger)); &\\
			n_2: ~O(KeepSafe(Roger))\Rightarrow O(\neg HardBrake(AV)); \\
			n_9: ~Passenger(Roger)\Rightarrow Citizen(Roger)\wedge Life(Roger); \\
			p_1:~\text{The safety of citizens' lives has the highest priority while driving. }
		\end{array}
		\right\}
	\end{equation*}\\
	\begin{equation*}
		Exp(O(HitTree(AV)\wedge \neg Damage(Lamp)))    
		=Exp(O(\neg HardBrake(AV))) 
	\end{equation*}

	\begin{equation*}
		\cup\left\{
		\begin{array}{ll}
			n_3: Life(Pongo)\Rightarrow O(\neg Hurt(Pongo))   &\\
			n_5: O(\neg Hurt(Pongo))\wedge P(\neg HardBrake(AV))\Rightarrow O(Swerve(AV))
			&\\
			n_6: O(Swerve
			(AV))\Rightarrow O(HitTree(AV)\wedge \neg Damage(Lamp))\\
			n_8: Dog(Pongo)\Rightarrow Life(Pongo)\\
			p_2: \text{ While safe, the value of an animal's life should be higher than the value of property.}\\
			p_2<p_1
		\end{array}
		\right\}
	\end{equation*}
\end{footnotesize}

For the argument leading to the conclusion ``$O(\neg HardBrake(AV))$'', the norms it uses include $n_1$ and $n_2$, and the last principle is $p_1$. This argument defeats the conflicting argument (with the conclusion $O(HardBrake(AV))$) and provides a defense for itself, so according to Definition \ref{def:XAI}, the explanation for accepting $O(\neg HardBrake(AV))$ includes $n_1$, $n_2$ and $p_1$. 
This explanation could be understood as: in this case the autonomous vehicle cannot choose to take a hard brake, because this would harm the passengers. And the law prohibits harming the lives of citizens.

For the conclusion ``$O(HitTree(AV)\wedge \neg Damage(Lamp))$'', since according to $r_1$, $O(\neg HardBrake(AV))$ is necessary to deduce one of the antecedents of $n_5$, and the argument for it defends itself, thus defends the whole argument for the conclusion $O(HitTree(AV)\wedge \neg Damage(Lamp))$. Therefore $Exp(O(\neg HardBrake(AV)))$ is a subset of the whole explanation. In addition, excluding strict rules, the relevant norms include $n_3$, $n_5$ and $n_6$, and the last principles also include $p_2$, hence the relevant priority relation is included. 

In fact, in this case, if there is no rule $r_5$, the option \(O(HitLamp(AV)\wedge \neg Damage(AV))\) may still appear in another set of actions that can be collectively accepted after argument evaluation. However, the existence of the $Pref$ relation and $r_5$ leads to an inconsistency in the strict rule closure of this set of actions, thereby making this output unacceptable. 

This explanation could be understood as: in this case, the autonomous vehicle will choose to hit a tree rather than hard braking, hitting an animal, or hitting a street lamp. This decision is based on the previously illustrated explanation of why hard braking is not an option, as well as the legal principle of prioritising the safety of citizens ($p_1$) over the ethical principle of protecting animals ($p_2$). Additionally, for the two actions with personal preference, the manufacturer prefers to hit the tree. This maybe because hitting the street lamp will lead to more follow-up problems. But the designer does not have to give a more detailed description. This can be understood as, apart from the law and public values, there is room for the manufacturer's discretion.

\section{Conclusion}\label{sec:conc}

This paper introduces deontic logic operators into \lsc, offering autonomous vehicle designers and manufacturers a comprehensive legal explanation. This explanation clarifies how a designed action complies with the law, why it is accepted over conflicting alternatives, and how underlying legal or ethical principles are applied. By utilizing deontic logic, it also provides clear guidance on the permissions and obligations related to specific actions. This legal explanation supports the design process, helps in obtaining legal approvals, avoids penalties, and builds public trust. 

Our future work will focus on deepening the connection between design patterns and legal norms to provide more effective legal support for designers. In addition, further consideration of the case we've shown, after a more detailed definition, or adding more helpful and useful logical components (such as the addition of temporal operators to express the evolution of concepts based on the timeline and context change), from the perspective of epistemic reasoning, we can determine that the argument that applies the norm $AIdriver(Roger)\Rightarrow\neg Life(Roger)$ should be preferred over the argument that applies the norm $Driver(Roger)\Rightarrow Life(Roger)$, and can therefore strictly defeat the former. 
In the context of AI applications targeted by \lesac, new technologies and innovations expand our conceptual understanding, which may lead to conflicts in legal classifications. While passengers are typically considered as citizens or lives, how should AI robot passengers be classified? Current legal principles may no longer be appropriate in such cases. This suggests the need to go beyond deontic operators and incorporate additional logical approaches to better address conflict and decision-making, as well as provide more nuanced explanations based on timelines, evolving legal and ethical principles, and conceptual classifications. These considerations will guide our future work.

\bibliographystyle{abbrv}
\bibliography{bib}

\end{document}